\documentclass[a4paper,11pt]{article}

\usepackage{mystyle}

\usepackage{times}
\usepackage{graphicx} % more modern
\usepackage{subfigure}

\usepackage{natbib}

\usepackage{algorithm}
\usepackage{algorithmic}
\usepackage[draft]{hyperref}

\usepackage{units}

\usepackage{dsfont}
\usepackage{bm}

\usepackage{authblk}
\usepackage{fullpage}

\title{Stochastic Wasserstein Barycenters}
\author{Sebastian Claici}
\author{Edward Chien}
\author{Justin Solomon}
\affil{
    Computer Science and Artificial Intelligence Laboratory\\
    Massachusetts Institute of Technology\\
    \{sclaici, edchien, jsolomon\}@mit.edu
}
\date{}

\begin{document}
\maketitle
\begin{abstract}
  We present a stochastic algorithm to compute the barycenter of a set of probability distributions under the Wasserstein metric from optimal transport. Unlike previous approaches, our method extends to continuous input distributions and allows the support of the barycenter to be adjusted in each iteration. We tackle the problem without regularization, allowing us to recover a much sharper output. We give examples where our algorithm recovers a more meaningful barycenter than previous work. Our method is versatile and can be extended to applications such as generating super samples from a given distribution and recovering blue noise approximations.% to an input image.%<-- removed to avoid sounding like image processing
\end{abstract}

% !TEX root = icml2018_barycenters.tex

\section{Introduction}
\label{sec:intro}

Several scenarios in machine learning require summarizing a collection of probability distributions with shared structure but individual bias.  For instance, multiple sensors might gather data from the same environment with different noise distributions; the samples they collect must be assembled into a single signal.  As another example, a dataset might be split among multiple computers, each of which carries out MCMC Bayesian inference for a given model; the resulting ``subset posterior'' latent variable distributions must be reassembled into a single posterior for the entire dataset.  In each case, the summarized whole can be better than the sum of its parts:  noise in the input distributions cancels when averaging, while shared structure is reinforced.

%Aggregating several distributions into a single distribution on which inference can be performed cheaply or from which samples can be drawn is a long standing problem in machine learning (CITE). As an example, consider the problem of combining several subset posteriors of an MCMC sampler into a single distribution that best fits the data (CITE). % not sure i'd use the word `cheap' to refer to previous work since our algorithm is slow :-)

The theory of \emph{optimal transport} (OT) provides a promising and theoretically-justified approach to averaging distributions over a geometric domain.  OT equips the space of measures with a distance metric known as the Wasserstein distance; the average, or \emph{barycenter}, of a collection $\{\mu_j\}_{j=1}^N$ is then defined as a Fr\'echet mean minimizing the sum of squared Wasserstein distances to the input distributions~\cite{agueh_barycenters_2011}.  This mean is aware of the geometric structure of the underlying space.  For example, the Wasserstein barycenter of two Dirac distributions $\delta_x$ and $\delta_y$ supported at points $x,y\in\R^n$ is a single Dirac delta at the center point $\delta_{(x+y)/2}$ rather than the bimodal superposition $\frac{1}{2}(\delta_x+\delta_y)$  obtained by averaging algebraically.

%The theory of optimal transport is a principled approach to comparing and averaging distributions in the space of probability measures. To highlight the difference between optimal transport approaches and a simple Euclidean approach, consider two Dirac distributions $\delta_x$ and $\delta_y$ supported at points $x$ and $y$ in some metric space. The Euclidean average of the two is the bimodal distribution $(\delta_x + \delta_y) / 2$, while the Wasserstein barycenter is the unimodal distribution supported at $(x + y) / 2$.

%The typical way to think of the average of some distributions $\{\mu_j\}_{j=1}^N$ in our context is as a Frechet mean in the metric space of probability distributions endowed with the Wasserstein metric (CITE). This measure is called the Wasserstein barycenter of the input distributions.

%, but sparse, --- how is it sparse?
If the input distributions are discrete, then the Wasserstein barycenter is computable in polynomial time by solving a large linear program~\cite{anderes_discrete_2016}.  Adding entropic regularization yields elegant and efficient approximation algorithms~\cite{genevay_stochastic_2016,cuturi_smoothed_2016,cuturi_fast_2014,ye_fast_2017}. %However, these techniques scale poorly with the support of the measures, and scale poorly in dimension even if the input measures have small support. %<---- so does yours :-)
These and other state-of-the-art methods typically suffer from any of a few drawbacks, mainly (1) poor behavior as regularization decreases, (2) required access to the distribution functions rather than sampling machinery, and/or (3) a fixed discretization on which the input or output distribution is supported, chosen without knowledge of the barycenter's structure.

%When the distributions are continuous, the machinery of (CITE) can be used to obtain an approximation to the true barycenter supported on a finite set of points. However, the support of the barycenter must be chosen beforehand, and thus it is possible that a sampler from the barycenter be inaccurate simply because the true barycenter does not overlap significantly many support points from the approximation.
%To address these issues,
Given sample access to $N$ distributions $\mu_j$, we propose an algorithm that iteratively refines an approximation to the true Wasserstein barycenter. The support of our barycenter is adjusted in each iteration, adapting to the geometry of the desired output. Unlike most existing OT algorithms, we tackle the problem without regularization, yielding a sharp result. Experiments show that the support of our barycenter is contained (to tolerance) within the support of the true barycenter even though we use stochastic optimization rather than computational geometry.

%This paper presents an orthogonal approach to the barycenter problem. We give an algorithm that yields samples from the barycenter distribution. Our approach in arbitrary dimensions.

\paragraph{Contributions.} We give a straightforward parallelizable stochastic algorithm to approximate and sample from the Wasserstein barycenter of a collection of distributions, which does not rely on regularization to make the problem tractable.  We only employ samplers from the input distributions, and our technique is not restricted to input or output distributions supported on a fixed set of points. %Our algorithm is semi-discrete in nature, but can be applied to continuous distributions.
%To our knowledge, this is the first barycenter algorithm that does not require a fixed support for the output.  % <--- justin isn't totally confident about this, and plus we've made the point several times
We verify convergence properties and showcase examples where our approach is inherently more suitable than competing approaches that require a fixed support.

\section{Related Work}
\label{sec:related}

%\paragraph*{Computational OT.}
OT has made significant inroads in computation and machine learning; see~\cite{levy2017notions,peyre2018computational,solomon2018optimal} for surveys.  Although most algorithms we highlight approximate OT distances rather barycenters, they serve as potential starting points for barycenter computation.

\citet{cuturi_sinkhorn_2013} renewed interest in OT in machine learning through introduction of entropic regularization.  The resulting Sinkhorn algorithm is compact and efficient; it has been extended to barycenter problems through gradient descent~\cite{cuturi_fast_2014} or iterative projection~\cite{benamou_iterative_2015}.  Improvements for structured instances enhance Sinkhorn's efficiency, e.g.\ via fast convolution~\cite{solomon_convolutional_2015}  or multiscale approximation~\cite{schmitzer2016sparse}.
% could cite some applications of barycenters here

% semidiscrete --- basics, fernando

Our technique, however, is influenced more by \emph{semidiscrete} methods, which compute OT distances to distributions supported on a finite set of points.  Semidiscrete OT is equivalent to computing a power diagram~\cite{aurenhammer1987power,aurenhammer1992minkowski}, a weighted generalization of Voronoi diagrams.  Algorithms by~\citet{merigot2011multiscale} in 2D and~\citet{levy2015numerical} in 3D use computational geometry to extract gradients for the dual semidiscrete problem;~\citet{kitagawa2016convergence} accelerate convergence via a second-order Newton method.  Similar to our technique, \citet{de2012blue} move the support of a discrete approximation to a distribution to reduce Wasserstein distance.%; although they consider only one distribution instead of a barycenter and rely on analytical gradients, their mathematical starting point resembles ours.

% randomized
Recent stochastic techniques target learning applications.  \citet{genevay_stochastic_2016} propose a scalable stochastic algorithm based on the dual of the entropically-regularized problem; they are among the first to consider the setting of sample-based access to distributions but rely on entropic regularization to smooth out the problem and approximate OT distances rather than barycenters.  \citet{staib2017parallel} propose a stochastic barycenter algorithm from samples, but a finite, fixed set of support points must be provided a priori.  \citet{arjovsky2017wasserstein} incorporate a coarse stochastic approximation of the 1-Wasserstein distance into a generative adversarial network (GAN); the 1-Wasserstein distance typically is not suitable for barycenter computation.

% learning with OT distances
Further machine learning applications range from supervised learning to Bayesian inference. \citet{DBLP:journals/corr/abs-1708-01955} leverage OT theory for dictionary learning. \citet{DBLP:conf/icml/CarriereCO17} apply the Wasserstein distance to point cloud segmentation by developing a notion of distance on topological persistence diagrams. \citet{courty_optimal_2016} utilize the optimal transport plan for transfer learning on different domains. \citet{srivastava_wasp:_2015,srivastava_scalable_2015} use the Wasserstein barycenter to approximate the posterior distribution of a full dataset by the barycenter of the posteriors on smaller subsets; their method provably recovers the full posterior as the number of subsets increases.

%%% Local Variables:
%%% mode: latex
%%% TeX-master: "icml2018_barycenters"
%%% End:

% !TEX root = icml2018_barycenters.tex

\section{Background and Preliminaries}
\label{sec:background}

Let $(X, d)$ be a metric space, and let $\mathcal{P}(X)$ be the space of probability measures on $X$ with finite second moment. Given two measures $\mu_1, \mu_2 \in \mathcal{P}(X)$, the squared 2-Wasserstein distance between $\mu_1$ and $\mu_2$ is given by
\begin{align}
  \label{eq:otproblem}
  W_2^2(\mu_1, \mu_2) = \left( \inf_{\gamma \in \Gamma(\mu_1, \mu_2)} \int_{X \times X}\hspace{-.2in}d(\*x, \*y)^2\,\mathrm{d} \gamma(\*x, \*y) \right).
\end{align}
Here, $\Gamma(\mu_1, \mu_2) \subset \mathcal{P}(X \times X)$ is the set of measure couplings between $\mu_1$ and $\mu_2$:
\begin{align*}
  \Gamma(\mu_1, \mu_2)\!=\!\left\{\gamma \in \mathcal{P}(X\!\times\!X) : (\pi_x)_{\#}\gamma\!=\!\mu_1, (\pi_y)_{\#} \gamma\!=\!\mu_2 \right\},
\end{align*}
where $\pi_x$ and $\pi_y$ are the two projections of $X \times X$, and the push-forward of a measure through a measurable map is defined as $f_{\#}\mu(A) = \mu(f^{-1}(A))$ for any set $A$ in a $\sigma$-algebra of $X$.

For measures $\mu_1, \ldots, \mu_N$, we can define the Wasserstein barycenter as the minimizer of the functional
\begin{align}
  \label{eq:barycenter}
  F[\nu] = \frac{1}{N} \sum_{j=1}^N W_2^2(\nu, \mu_j).
\end{align}
When the input measures are discrete distributions, \eqref{eq:barycenter} is a linear program solvable in polynomial time.

If at least one of the measures $\mu_j$ is absolutely continuous with respect to the Lebesgue measure, then~\eqref{eq:barycenter} admits a unique minimizer $\mu^*$~\cite{agueh_barycenters_2011,santambrogio_optimal_2015}. However, $\mu^*$ will also be absolutely continuous, % w.r.t Lebesgue,
%and thus we can only represent a finite approximation of it.
implying that computational systems typically can only find an inexact finite approximation.

We study a discretization of this problem. Suppose $\Sigma \subset X$ consists of $m$ points $\{\*x^i\}_{i=1}^m$, % with $x^i \in X$, % <--- already have Sigma subset X
and define the functional
\begin{align}
  \label{eq:baryfinite}
  F[\Sigma] = \frac{1}{N} \sum_{j=1}^N W_2^2\left(\frac{1}{m}\sum_{i=1}^m \delta_{\*x^i}, \mu_j \right).
\end{align}

We define the main problem.
\begin{problem}[Semidiscrete approximation]
  \label{eq:problem}
  Find a minimizer of $\Sigma \to F[\Sigma]$ subject to the constraints $\Sigma \subset X$,  $|\Sigma| = m.$
\end{problem}

Solving problem~\eqref{eq:problem} for a single input measure is equivalent to finding the optimal $m$-point approximation to the input measure. We can use the solution as a set of supersamples from the input~\cite{DBLP:conf/uai/ChenWS10}, or if the input distribution is a grayscale image, the solution yields a blue noise approximation to the image~\cite{de2012blue}.

%%% Local Variables:
%%% mode: latex
%%% TeX-master: "icml2018_barycenters"
%%% End:

% !TEX root = icml2018_barycenters.tex

\section{Mathematical Formulation}

%We propose a semi-discrete algorithm for solving the barycenter problem.

%\subsection{Preliminaries}

The OT problem~\eqref{eq:otproblem} admits an equivalent dual problem
\begin{align}
  \label{eq:otdual}
  \sup_{\phi \in L^1(X)} \int_X \phi(\*x)\,\mathrm{d}\nu(\*x) + \int_X \overline{\phi}(\*y)\,\mathrm{d}\mu(\*y),
\end{align}
where $\phi$ is the Kantorovich potential and $\overline{\phi}(\*x) \eqdef \inf_{\*y \in X} \{d(\*x,\*y)^2 - \phi(\*y)\}$ is the $c$-transform of $\phi$~\cite{santambrogio_optimal_2015,villani_optimal_2009}. %Intuitively, one can think of the dual problem as trying to maximize profit for a merchant buying at point $x$ and selling at point $y$. The use of the $c$-transform encodes the implicit constraint that the profit obtained can't undercut the market.

%\subsection{Derivation}

Following \citet{santambrogio_optimal_2015}, if $\nu = \sum_{i=1}^m \frac{1}{m}\delta_{\*x^i}$ is a finite measure supported on $\Sigma = \{\*x^i\}_{i=1}^m$, then~\eqref{eq:otdual} becomes
\begin{align}
  \label{eq:otfinite}
  %F^{\mu} =
  \max_{\phi \in \mathbb{R}^m} \left\{\sum_i \frac{1}{m} \phi^i + \int_X \overline{\phi}(\*y)\,\mathrm{d}\mu(\*y)\right\},
\end{align}
where $\bm{\phi}=(\phi^1,\ldots,\phi^m)$.
Note that the function $\phi\in L^1(X)$ is replaced with a finite-dimensional $\bm{\phi}\in\mathbb R^m$.
%
%Note that our optimization over the space of 1-Lipschitz functions becomes an optimization over $\mathbb{R}^m$.

%We will write $F_{\phi}$ and $F_{\Sigma}$ for the functional $F$ as a function of the dual potentials and positions of the points, respectively.
With this formula in mind, define
\begin{equation}\label{eq:F}
F_{\mathrm{OT}}[\phi,\Sigma;\mu]:=\sum_i \frac{1}{m} \phi^i + \int_X \overline{\phi}(\*y)\,\mathrm{d}\mu(\*y).
\end{equation}
Constant shifts in the $\phi^i$ do not change the value of $F_{\mathrm{OT}}$. $F_{\mathrm{OT}}$ has a simple derivative with respect to the $\phi^i$'s:
\begin{align}
  \label{eq:gradwts}
  \frac{\partial F_{\mathrm{OT}}}{\partial \phi^i} = \frac{1}{m} - \int_{V_\phi^i} \mathrm{d}\mu(\*y)
\end{align}
where $V_\phi^i$ is the \emph{power cell} %(or generalized Voronoi region)
of point $\*x^i$:% defined by
\begin{align*}
  V_\phi^i = \{x \in X : d(\*x, \*x^i)^2 - \phi^i  \leq d(\*x, \*x^{i'})^2 - \phi^{i'}, \forall i'\}.
\end{align*}

From here on we work with compact subsets of the Euclidean space $\mathbb{R}^D$ endowed with the Euclidean metric, $d(\*x, \*y) = \|\*x - \*y\|_2$. To differentiate with respect to the $\*x^i$'s, notice that the first term in equation~\eqref{eq:F} does not depend on the positions of the points. We rewrite the second term as
$$\sum_{i=1}^m \int_{V_\phi^i} (d(\*y, \*x^i)^2 - \phi^i)\,\mathrm{d}\mu(\*y).$$
%We can take the derivative of this term
Using Reynolds' transport theorem to differentiate while accounting for boundary terms shows
\begin{align}
  \label{eq:gradpts}
  \frac{\partial F_{\mathrm{OT}}}{\partial \*x^i} = \*x^i \int_{V_\phi^i} \mathrm{d}\mu(\*y)  - \int_{V_\phi^i} \*y\,\mathrm{d}\mu(\*y).
\end{align}
Equation~\eqref{eq:gradwts} confirms the intuition that each cell contains as much mass as its associated source point. We will leverage~\eqref{eq:gradpts} to design a fixed-point iteration that moves each point to the center of its power cell.

 Each subproblem of~\eqref{eq:baryfinite} admits a different Kantorovich potential $\bm{\phi}_j =(\phi_j^1,\ldots,\phi_j^m)$, giving the following optimization functional%. The functional and derivatives are given by
\begin{equation}
  \label{eq:func}
  F\left[\{\bm{\phi}_j\}_{j=1}^N,\Sigma;\{\mu_j\}_{j=1}^N\right]\!=\!\frac{1}{N} \sum_{j=1}^N F_{\mathrm{OT}}[\bm{\phi}_j,\Sigma;\mu_j]
\end{equation}
Define %we introduce the following notation:
$$
  a^i_j = \int_{V_{\phi^i_j}}\,\mathrm{d}\mu(\*y) \hspace{.5in}
  b^i_j = \frac{1}{a^i_j}\int_{V_{\phi_j^i}} \*y\,\mathrm{d}\mu(\*y).
$$
With this notation in place, the partial derivatives are
\begin{equation}\label{eq:Fgrad}
  \frac{\partial F}{\partial \phi^i_j}\!=\!\frac{1}{N} \left(\frac{1}{m}\!-\!a^i_j\right)\hspace{.3in}
  \frac{\partial F}{\partial \*x^i}\!=\!\frac{1}{N}\sum_{j=1}^N a^i_j \left( \*x^i\!-\!b^i_j \right).
\end{equation}

\section{Optimization}

With our optimization objective function in place, we now introduce our barycenter algorithm. To simplify nomenclature, from here on we refer to the dual potentials $\phi_j$ as weights on the generalized Voronoi diagram. Our overall strategy is an alternating optimization of $F$ in~\eqref{eq:func}:
\begin{itemize}
\item For fixed point positions, $F$ is concave in the weights and is optimized using stochastic gradient ascent.
\item For fixed weights, we apply a single %the solution can be found by a  % <--- we're not sure this is true...
fixed point iteration akin to Lloyd's algorithm~\cite{lloyd1982least}.
\end{itemize}

\subsection{Estimating Gradients}

Each of $a^i_j$ and $b^i_j$ can be expressed as an expectation of a simple function with respect to the $\mu_j$. We estimate these quantities by a simple Monte Carlo scheme.

In more detail, we can rewrite $a^i_j$ and $b^i_j$ as
\begin{align*}
  a_j^i = \mathbb{E}_{y \sim \mu_j}\left[\mathds{1}_{y \in V_{\phi_j}^i}\right] \hspace{.5in}
  b_j^i = \mathbb{E}_{y \sim \mu_j}\left[y \cdot\mathds{1}_{y \in V_{\phi_j}^i} \right].
\end{align*}
Here, $\mathds1$ indicates the indicator function of a set.

Since we have sample access to each $\mu_j$, the expectations can be approximated by drawing $K$ points independently $y_k \sim \mu_j$ and computing
\begin{align}
  \label{eq:approxab}
  \hat{a}_j^i = \frac{1}{K} \sum_{k=1}^K \mathds{1}_{y_k \in V_{\phi_j}^i} \hspace{.5in}
  \hat{b}_j^i = \frac{1}{K} \sum_{k=1}^K y_k \cdot\mathds{1}_{y_k \in V_{\phi_j}^i}.
\end{align}

\subsection{Concave Maximization}

The first step in our alternating optimization maximizes $F$ over the weights $\phi \in \mathbb{R}^m$ while the points $\*x^i$ are fixed. We call this step of the algorithm an \emph{ascent} step.

For a fixed set of points, the functional $F$ is concave in the weights $\phi_j$, since it is the dual of the convex semidiscrete transport problem. %$F$ is a sum of an affine function and a collection of infima of linear functions of $\phi_j$.
To solve for the weights, we perform gradient ascent using the formula in~\eqref{eq:Fgrad} where $a_j^i$ is approximated using $\hat a_j^i$. The gradient for a set of weights $\phi_j$ only requires computation of the density of a single measure $\mu_j$, implying that the ascent steps can be decoupled across different measures.

Write $w^0 = \phi_j$ for the initial iterate. The simplest version of our algorithm updates %We perform the following update
\begin{align*}
  w^{k+1} = w^k + \alpha \frac{\partial F}{\partial \phi_j}[w^k].
\end{align*}
The iterates converge when each point contains equal mass in its associated power cell.

%We know that
$F$ has a known Hessian as a function of the $\phi_j$ that can be used in Newton's algorithm~\cite{kitagawa_convergence_2016}. Computing the Hessian, however, is only possible with access to the density functions of the $\mu_j$'s as it requires computing a density of the measure on the boundary between two power cells. The boundary set is inherently lower dimensional than the problem space, and hence sample access to the $\mu_j$ is insufficient. Moreover, even had we access to the probability density functions, computing the Hessian would require the Delaunay triangulation of the point set, which is expensive in more than two dimensions. % <-- doesn't have anything to do with step size...

In any event, choosing the step size $\alpha$ is important for convergence.  Line search is difficult as we do not have access to true objective value at each iterate.  Instead, we rely on Nesterov acceleration to improve performance~\cite{nesterov1983method}. With acceleration, our iterates are
\begin{align}
  \label{eq:ascendStep}
    z^{k+1} = \beta z^k + \frac{\partial F}{\partial \phi_j}[w^k]\\
    w^{k+1} = w^k + \alpha z^{k+1}.
\end{align}
%Typical values that we have found are
where $w^k, z^k \in \mathbb{R}^m$. In our experiments, we use $\alpha = 10^{-3}$ and $\beta = 0.99$. Convergence of the accelerated gradient method can be shown when $\alpha = \nicefrac{1}{L}$ where $L$ is the Lipschitz constant of $F$; in \S\ref{sec:convergence}, we give an estimate of this constant. Our convergence criterion for this step is $\|\nabla F\|_2^2 \leq \epsilon$.

\subsection{Fixed Point Iteration}

The second step of our optimization is a fixed point iteration on the point positions.  This step is similar to the point update in a $k$-means algorithm in that it snaps points to the centers of local cells, and we refer to it as a \emph{snap} step.

We set the second gradient in~\eqref{eq:Fgrad} to zero:
\begin{align*}
  &\frac{\partial F}{\partial \*x^i} = 0 \hspace{0.3in}\Longrightarrow
  &\frac{1}{N}\sum_{j=1}^N a_j^i(\*x^i - b_j^i) = 0
\end{align*}
which leads to the point update
\begin{align}
  \label{eq:snapStep}
  \*x^i = \frac{\sum_{j=1}^N a_j^i b_j^i}{\sum_{j=1}^Na_j^i}.
\end{align}

This suggests a fixed point iteration for the $\*x^i$'s that can be decomposed into the following steps:
\begin{enumerate}
\item First find the barycenter of the power cells of each $\*x^i$ with respect to each $\mu_j$.
\item Then, average the points with weights given by the density of each measure in the cell.
\end{enumerate}
%\justin{In the two steps above, be sure to mention what these words mean in terms of the $a$'s and $b$'s}

If the concave maximization has converged appropriately, and uniform areas $a^i_j$ have been achieved, then the update step becomes a uniform average over the barycenters $b^i_j$ with respect to each measure.

\subsection{Global and Local Strategies}

\begin{algorithm}[tb]
  \caption{Optimizing estimate of barycenter support}
  \label{alg:iterate}
  \begin{algorithmic}[1]
    \REQUIRE Estimate of barycenter support $\Sigma = \{\*x_i\}_{i=1}^m$
    \ENSURE Optimized barycenter support $\Sigma^*$ with lower cost.
    \FOR{$t = 1, 2, \ldots, T$}
    \FOR{$j = 1, 2, \ldots, J$}
    \STATE $z^0 \gets 0$\hfill\COMMENT{Ascent on weights}
    \STATE $w^0 \gets \phi_j$
    \WHILE{$\left\|\frac{\partial F}{\partial \phi_j}\right\| > \epsilon$}
    \STATE Compute $\hat{a}_j^i$ according to equation~\eqref{eq:approxab}
    \STATE $z^{k+1} = \beta z^k + \frac{\partial F}{\partial \phi_j}[w^k]$
    \STATE $w^{k+1} = w^k + \alpha z^{k+1}$
    \ENDWHILE
    \STATE $\phi_j \gets w^{\mathrm{end}}$
    \ENDFOR
    \STATE Compute $\hat{b}_j^i$ according to equation~\eqref{eq:approxab}
    \FOR{$\*x_i \in S$}
    \STATE $\*x_i \gets \frac{\sum_{j=1}^N \hat{a}_j^i \hat{b}_j^i}{\sum_{j=1}^N \hat{a}_j^i}$    \hfill\COMMENT{Snap points}
    \ENDFOR
    \ENDFOR
  \end{algorithmic}
\end{algorithm}

The \emph{ascent} and \emph{snap} steps can be used to refine a configuration of points $\Sigma$. Once the iterates converge, we have an $m$-point approximation to the barycenter that can be used as an initialization for $m + 1$ point approximation in two ways. A new point $\*x$ is sampled uniformly from $X$, and then we have a choice between (1) moving all points including the new one or (2) allowing only $\*x$ to move.

These two approaches are codified in Algorithm~\ref{alg:iterate} where the choice on the set $S$ dictates which points move. The number of iterations of the outer loop is fixed beforehand. Typically, we see convergence in fewer than $20$ steps, and empirically, we observe good performance even with $T = 1$. The two most natural choices for $S$ are $S = \Sigma$ and $S = \{\*x\}$.  If the barycenter is absolutely continuous with respect to the underlying Lebesgue measure, these two strategies converge at the same rate asymptotically~\cite{brancolini2009long}.  The latter, however, can generate spurious samples that are not in the support of the barycenter. Optimizing the weights is regardless a global problem as moving or introducing points changes the volumes of the power cells of neighboring points.

Both algorithms are highly parallelizable, since (1) the gradient estimates are expectations computed using Monte Carlo integration and (2) the gradient step in the weights decouples across distributions.

%%% Local Variables:
%%% mode: latex
%%% TeX-master: "icml2018_barycenters"
%%% End:

% !TEX root = icml2018_barycenters.tex

\section{Analysis}
\label{sec:convergence}

We justify the use of uniform finitely-supported measures, and then prove that our algorithm converges to a local minimum cost under mild assumptions.

We assume in this section %for the purposes of this section
that at least one of the distributions $\mu_j$ is absolutely continuous with respect to the Lebesgue measure, ensuring a unique Wasserstein barycenter.

\subsection{Approximation Suitability}

The simplest approach for absolutely continuous measures $\mu_j \in \mathcal{P}(X)$ is to sample $p$ points from each of the $J$ measures and solve for the true barycenter of the empirical distributions~\cite{anderes_discrete_2016}. This approach likely approximates the barycenter as the number of samples increases, but requires solution of a linear program with $O(p^J)$ variables. As an alternative, \citet{staib2017parallel} propose a stochastic problem for approximating barycenters. They are able to prove a rate of convergence, but the support of their approximate barycenter is fixed to a finite set of points.

Our technique allows the support points to move during the optimization procedure, empirically allowing a better approximation of the barycenter with fewer points.  The following theoretical result shows that the use of uniform measures supported on a finite set of points can approximate the barycenter arbitrarily well: %\cite{kloeckner_approximation_2012,brancolini2009long} shows that a finitely-supported uniform measure is sufficient to approximate the absolutely continuous barycenter:
\begin{theorem*}[Metric convergence, \citet{kloeckner_approximation_2012,brancolini2009long}]
  Suppose $\nu_m^*$ is a uniform measure supported on $m$ points that minimizes $\frac{1}{N}\sum_{j=1}^N W_2(\nu_m^*, \mu_j)$, and let $\bar{\mu}$ denote the true barycenter of the measures $\{\mu_j\}_{j=1}^N$. Then $W_2(\nu_m^*, \bar{\mu}) \leq Cm^{-1/D}$ where $C$ depends on the space $X$, the dimension $D$, and the metric $d(\cdot, \cdot)$.
\end{theorem*}
%Since the Wasserstein distance metrizes weak convergence on the space of probability measures on a compact subset of $\mathbb{R}^D$, we have the following corollary:
%\begin{corollary}[Convergence in probability, \citet{santambrogio_optimal_2015,villani_optimal_2009}]
%  The measure $\nu_m^*$ weakly converges to $\bar{\mu}$ as $m$ increases: $\nu_m^* \rightharpoonup \bar{\mu}$.
%\end{corollary}
%See also~\cite{brancolini2009long} for a direct proof using $\Gamma$-convergence in the $N = 1$ case.
Note that this shows convergence in probability $\nu_m^* \rightharpoonup \bar{\mu}$ since the Wasserstein distance metrizes weak convergence~\cite{villani_optimal_2009}. \citet{brancolini2009long} also show asymptotic equivalence of the local and global algorithms.

While we cannot guarantee that our method converges to $\nu^*_m$, these properties indicate that the \emph{global} minimizer of our objective provides an effective approximant to the true barycenter as the number of support points $m\to\infty$.

\subsection{Algorithmic Properties}

%we can show convergence of the two step process to a minimum of the functional $F$. Empirically, we show convergence to the true barycenter in cases where it is known analytically.

Under mild assumptions on the $\mu_j$ (absolute continuity wrt Lebesgue), the functional $F$ is concave in the weights $\phi^j_i$ with fixed point positions, and in fact strictly concave up to constant shifts. We can investigate the convergence properties of the gradient ascent step of the algorithm. We assume in the following section that the partial derivatives are obtained exactly, rather than approximated via sampling, so our results will hold true in the limit, as number of samples increases. We show first that the gradient of $F$ is not necessarily Lipschitz continuous.

%It suffices to prove convergence in the case $N = 1$.

\begin{counterexample}
  Assume $X$ is a compact subset of $\mathbb{R}^D$. There are measures $\mu \in \mathcal{P}(X)$ for which the gradient of $F$ is not Lipschitz continuous. A set of weights that satisfies $ \frac{\partial F}{\partial \phi} = 0$ may not exist, and if it does, it may not be unique.
\end{counterexample}
\begin{proof}[Construction]
  We provide a counterexample for $D=1$. Let $X = [-1, 1]$ with the standard metric and $\mu = \delta_{0}$. Let $\Sigma = \{-1, 1\}$ be the fixed positions, and take $\phi_1 = \{-\epsilon, 0\}$ and $\phi_2 = \{\epsilon, 0\}$ for small $\epsilon$. Then $\|\phi_1 - \phi_2\|_1 = 2\epsilon$, but $\|\nabla F_{\phi}[\phi_1] - \nabla F_{\phi} [\phi_2]\|_1 = 2$.

  Non-existence is shown in Figure~\ref{fig:nonexist}. To see non-uniqueness, take $\mu = \frac{1}{2} \delta_{-\epsilon} + \frac{1}{2} \delta_{\epsilon}$ with $\Sigma$ as before. Any set of weights in $(-\epsilon, \epsilon)^2$ minimizes $F_{\phi}$.
\end{proof}

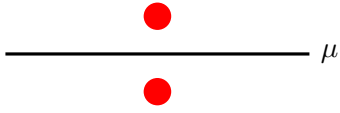
\begin{figure}[tb]
  \centering
  \begin{tikzpicture}
    \node[circle,fill=red] at (0,0.5) {};
    \node[circle,fill=red] at (0,-0.5) {};
    \draw[very thick] (-2,0) --  (2,0) node[right] (C) {$\mu$};
  \end{tikzpicture}
  \caption{Non-existence of a set of weights. Let $\mu$ be the uniform measure on the line segment, and $\Sigma$ be the two red points such that the line between them is orthogonal to the support of $\mu$. There is no set of weights such that the mass of $\mu$ is split evenly between the two red points.}
  \label{fig:nonexist}
\end{figure}

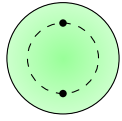
\begin{figure}[tb]
  \centering
  \begin{tikzpicture}[scale=0.7]
    \filldraw[even odd rule, inner color=green!40, outer color=green!20] (0,0) circle (30pt);
    \draw[dashed] (0,0) circle (19pt);
    \fill (0,19pt) circle (2pt);
    \fill (0,-19pt) circle (2pt);
  \end{tikzpicture}
  \caption{Non-unique minimizer on two points for the uniform measure defined on the unit disk. All antipodal points on the dashed circle at distance $2/\pi$ from the center are valid minimizers.}
  \label{fig:nonuniquemin}
\end{figure}

For mildly behaved measures $\mu$ the gradient of $F$ with respect to $\phi$ is Lipschitz continuous:
\begin{lemma}
Assume $X$ is a compact subset of $\mathbb{R}^D$, and $\mu$ is absolutely continuous with respect to the Lebesgue measure, with density function $\rho$. If the $m$ points of $\Sigma$ are distinct and $\rho \leq M$ almost everywhere for some constant $M$, then:
  \begin{align*}
    \|\nabla F_{\phi}[\phi_1] - \nabla F_{\phi} [\phi_2]\|_2 \leq \sqrt{m}\frac{M S}{2L} \|\phi_1 - \phi_2\|_2.
  \end{align*}
where $S$ denotes the surface area of $\partial \mathrm{conv}(X)$ and $L$ denotes the minimum pairwise distance between points in $\Sigma$.
\end{lemma}

\begin{proof}
  Consider the $i$th component of the gradient:
  \begin{align*}
    \left|\frac{ \partial F_{\phi}}{\partial \phi^i}[\phi_1] - \frac{ \partial F_{\phi}}{\partial \phi^i}[\phi_2] \right| &= \left|  \int_{V^i_{\phi_1}} \rho \, \mathrm{d}\lambda - \int_{V^i_{\phi_2}} \rho \, \mathrm{d}\lambda \right| \\
    &\leq \frac{S\| \phi_1 - \phi_2 \|_2}{2L}M.
  \end{align*}
The second inequality follows as the area of a power cell is bounded by $S$ and the faces of the cells change at a rate linear in $\| \phi_1 - \phi_2 \|_2$. The rate is dependent on the distance between the points, so the constant $L$ is required. The Lipschitz bound follows directly from considering all components of the gradient difference together.
\end{proof}

This lemma implies convergence for a step size that is the inverse of the Lipschitz constant. While the above requires absolute continuity of $\mu$, we have found that our ascent steps and method often converge even when this is not satisfied (see Figures~\ref{fig:sharp} and \ref{fig:ellipse}).

%To show convergence of the fixed point iteration, we use the following theorem from~\cite{trove.nla.gov.au/work/21124597}:
%
%\begin{theorem}
%  \label{thm:convergepts}
%  Let $\mathcal{A}$ be an algorithm that is monotonic with respect to a function $F$. Given an initial point $x_0$, suppose that the algorithm generates a sequence $\{x_k\}$ that lies in a compact set. Then the sequence has at least one accumulation point $\bar{x}$, and $J(\bar{x}) = \lim_{k \to \infty} J(x_k)$.
%\end{theorem}
%
%Our iterates for a fixed number of points $m$ can be seen as elements of $X^m \times \mathbb{R}^m$ as pairings of points $(\*x_1, \ldots, \*x_m) \in X^m$ and weights $(\phi_1, \ldots, \phi_m) \in \mathbb{R}^m$. The algorithm will monotonically decrease the transport cost at each iteration. What is left is to show that the sequence generated by the algorithm belongs to a compact set. But not that $X$ is compact, and hence $X^m$ is compact. Observe that no weight in absolute value will exceed $\text{diam}(X)$ (up to constant shifts in all weights). This shows that the iterates are all in a compact set, and theorem~\ref{thm:convergepts} applies. We thus have
%
%\begin{corollary}
%  The iterations in equations \eqref{eq:ascendStep} and \eqref{eq:snapStep} converge to a minimum of $F$.
%\end{corollary}
%
%This does not imply that the fixed point iteration itself will converge, as there may be several minimizers (see Figure~\ref{fig:nonuniquemin}). We have not been able to give an example where convergence does not occur, and empirically we achieve convergence in all of our experiments.

We may also show that our algorithm monotonically decreases $F[\Sigma]$ (defined in Equation~\eqref{eq:baryfinite}) after each pair of snap and then ascent steps for compact domain and absolutely continuous $\mu_j$. Recall that the transport cost for a map $T: X \to \Sigma$ sending measure $\mu_j$ to $\frac{1}{m}\sum_i \delta_{\mathbf{x}^i}$  is:
\[ \int_X d(x,T(x))^2 \, d \mu_j. \]
Fixing the power cells $V^i_j$ after an ascent step, we define $T_j(\Sigma)$ to be the transport cost for the map sending the power cells $V^i_j$ to the point set $\Sigma$, and $TC(\tilde{\Sigma}) = \frac{1}{N} \sum_j T_j$ to be the joint (average) transport cost. Letting $\tilde{\Sigma} = \{ \mathbf{\tilde{x}^i} \}$ denote the new positions after a snap step, we may now show:

\begin{lemma}
For $X \subset \mathbb{R}^D$ compact, and $\mu_j$ absolutely continuous with respect to the Lebesgue measure for all $j$:
  \begin{equation*}
    F[\tilde{\Sigma}] \leq F[\Sigma].
  \end{equation*}
\end{lemma}

\begin{proof}
 By strong duality, we have the following equality for each $j$ when the $\phi$ have been optimized after an ascent step:
\begin{equation*}
F_{OT}[\phi, \Sigma; \mu_j] = W^2_2\left( \frac{1}{m} \sum^m_{i=1} \delta_{\mathbf{x}^i}, \mu_j \right).
\end{equation*}
This implies that $F[\Sigma] = TC(\Sigma)$ as $W^2_2$ is simply the optimal transport cost. We now argue that $TC(\tilde{\Sigma}) \leq TC(\Sigma)$. We may split up the integrals for transport cost over the power cells corresponding to each $i$th point. We differentiate $\sum^N_{j = 1} \int_{V^i_j} \|x - p\|^2 \, d \mu_j$ with respect to $p$ to find the point with lowest joint transport cost to the cells $V^i_j$. Setting this to 0 yields $\sum_{j=1}^N a_j^i b^i_j - a^i_j p  = 0$.

Note this is equivalent to the barycenter update step in Equation~\eqref{eq:snapStep}, and with convergence of the previous ascent step, we should have uniform $a^i_j$ weights. This demonstrates that snapping to the uniform average of barycenters lowers $TC$, and we have that $F[\Sigma] = TC(\Sigma) \geq TC(\tilde{\Sigma}) \geq F[\tilde{\Sigma}]$. The last inequality follows as the next ascent step will find the optimal transport and decrease the transport cost.
\end{proof}

With joint transportation cost being non-negative, our objective function converges to a local minimum. This does not imply that our iterates converge, as there may not be a unique minimizing point configuration (see Figure~\ref{fig:nonuniquemin}). Empirically, our iterates converge in all of our test cases. We note also that our formula bears some resemblance to the mean-shift algorithm and to Lloyd's algorithm, both of which which are also known to converge under some assumptions \cite{li2007note,bottou1995convergence}.

\begin{figure*}[!t]
  \centering
  \begin{tabular}{@{}ccc@{}}
    \includegraphics[width=.25\textwidth]{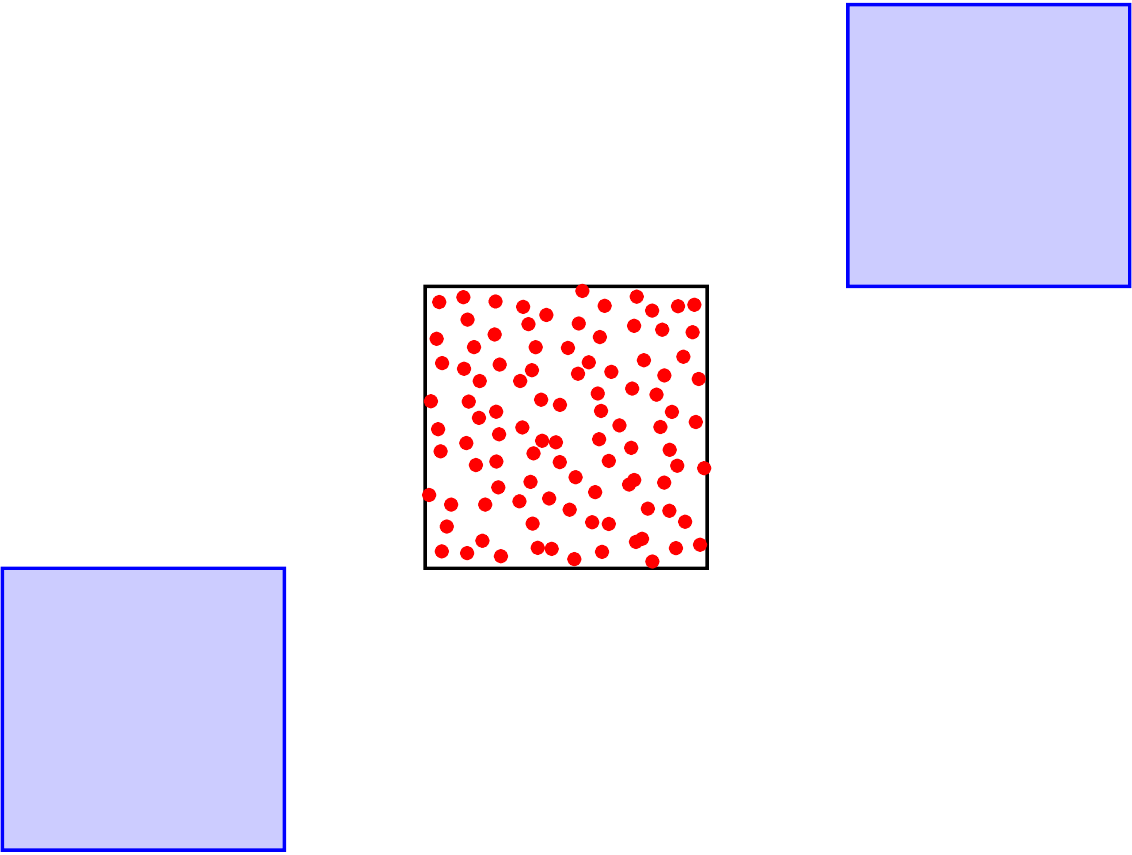} &
    \includegraphics[width=.25\textwidth]{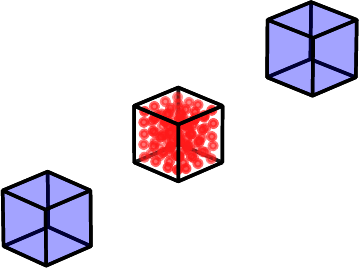} &
    \includegraphics[width=.25\textwidth]{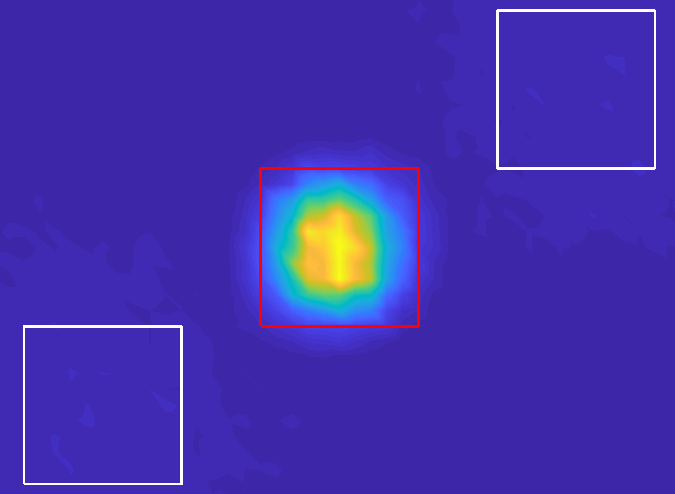}\\
    (a) & (b) & (c)
  \end{tabular}
  \caption{Barycenter when $N = 2$ tested on two uniform distributions over unit squares. (a) Our output: the input distributions are shown in blue, while the output barycenter points are shown in red, with the limits of the true barycenter in black. (b) A similar example in three dimensions. (c) The output barycenter of \cite{staib2017parallel}: note the output has non-zero measure outside the true barycenter.}
  \label{fig:mccann}
\end{figure*}

\begin{figure*}[!t]
  \centering
  \begin{tabular}{@{}ccc@{}}
    \includegraphics[width=.23\textwidth]{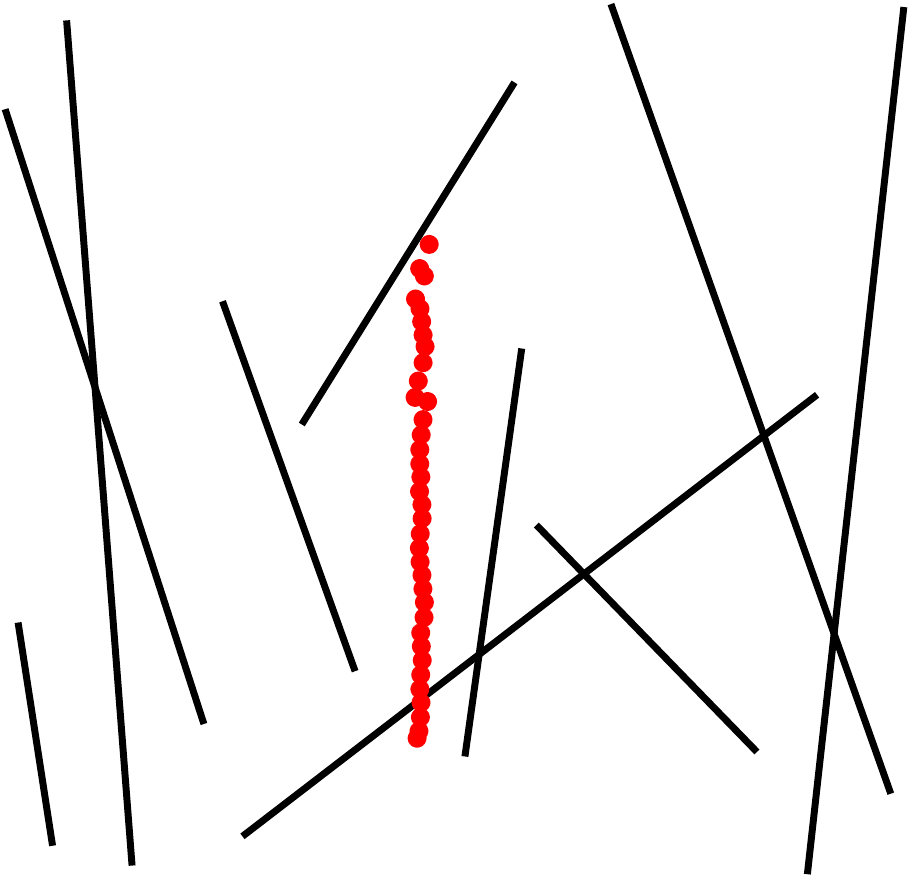} &
    \includegraphics[width=.23\textwidth]{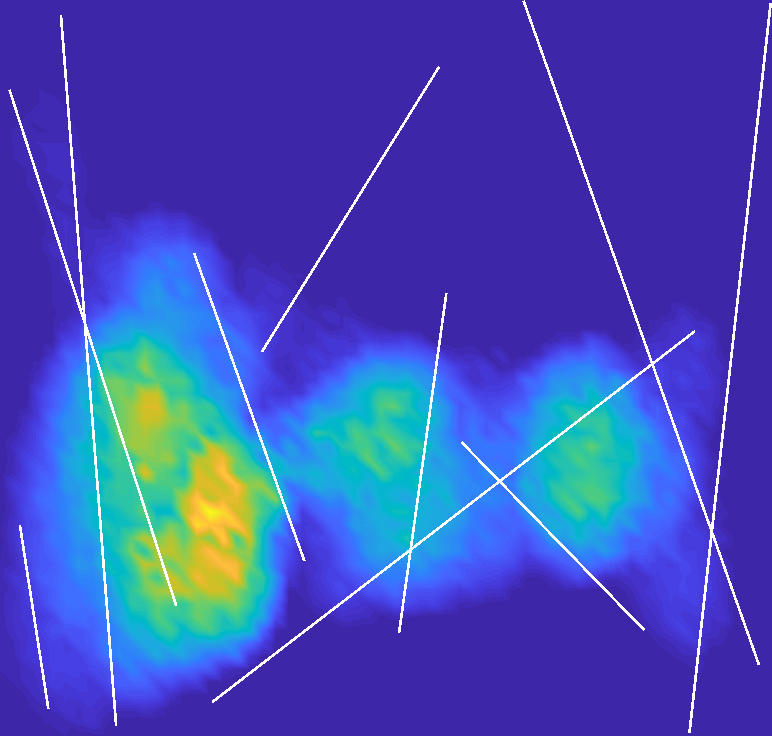} &
    \includegraphics[width=.23\textwidth]{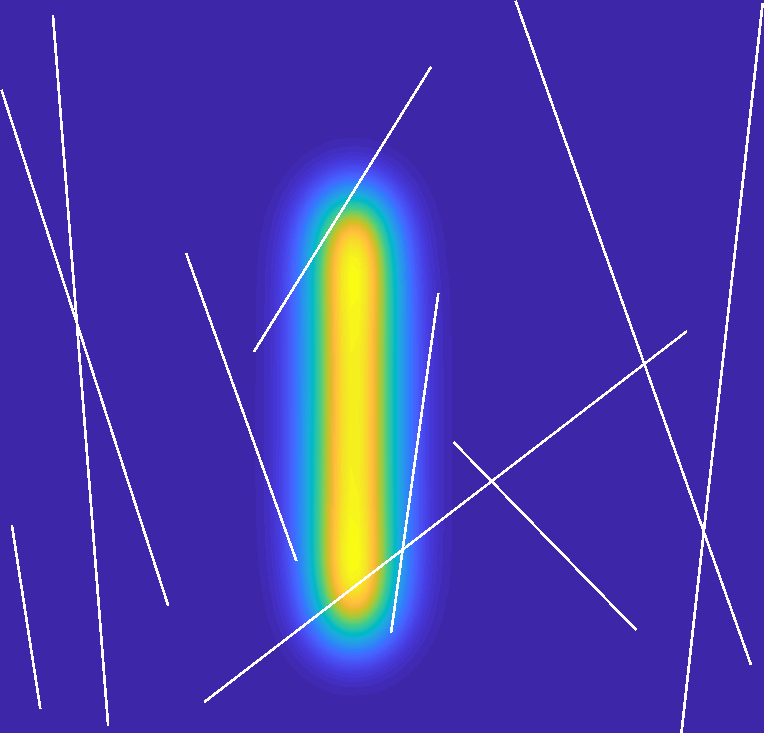}\\
    (a) & (b) & (c)
  \end{tabular}
  \caption{Barycenter of sharp featured distributions. (a) 50 points from our algorithm yields a barycenter supported on a line. (b) The barycenter from~\cite{staib2017parallel} using a grid of 20000 points. (c) Barycenter from~\cite{solomon_convolutional_2015} using a regularizer value of $\gamma = 0.1$; smaller regularizers were numerically unstable.}
  \label{fig:sharp}
\end{figure*}

\begin{figure*}[!t]
  \centering
  \begin{tabular}{@{}cccc@{}}
    \includegraphics[width=.21\textwidth]{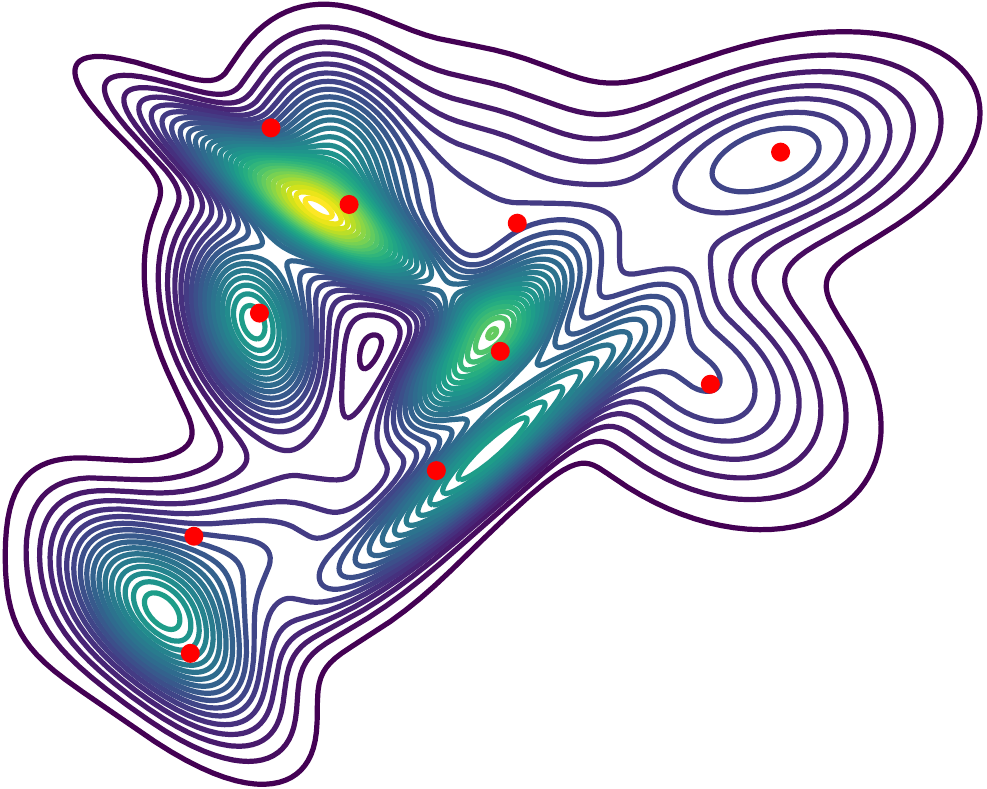} &
    \includegraphics[width=.21\textwidth]{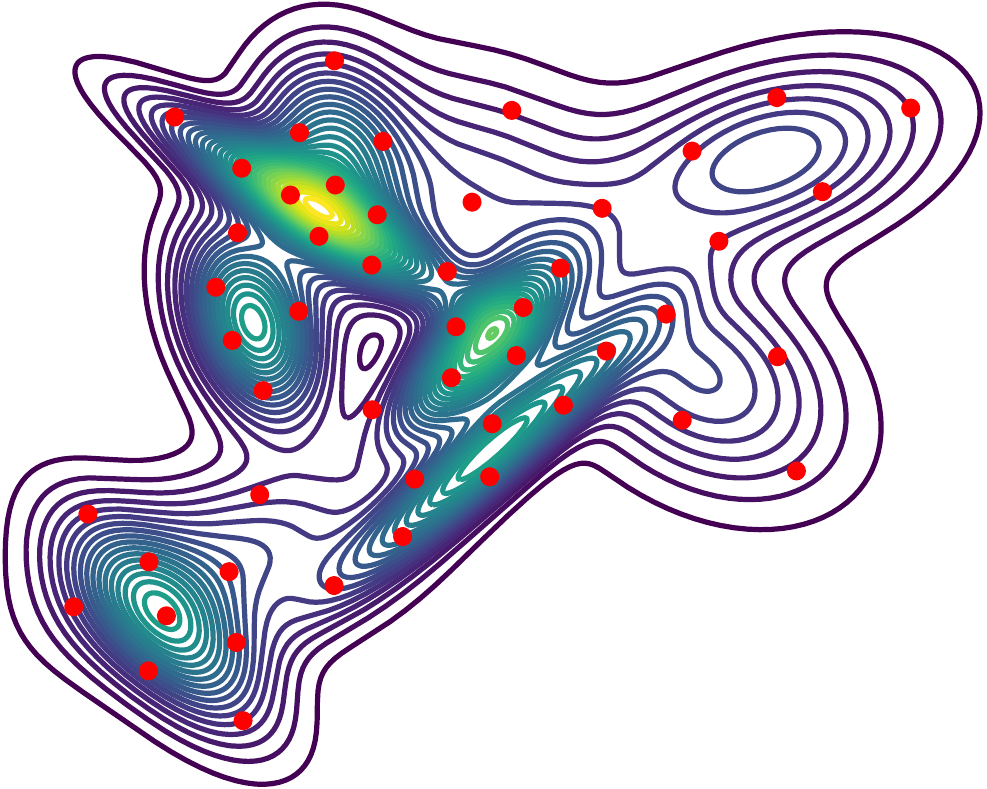} &
    \includegraphics[width=.21\textwidth]{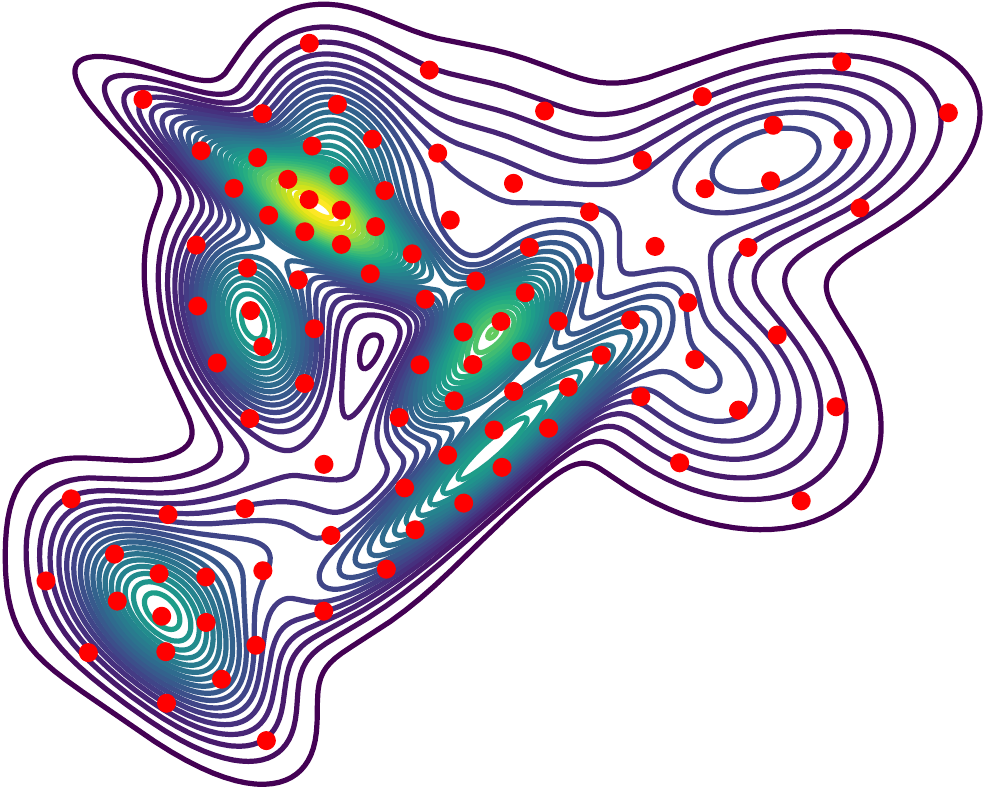} &
    \includegraphics[width=.21\textwidth]{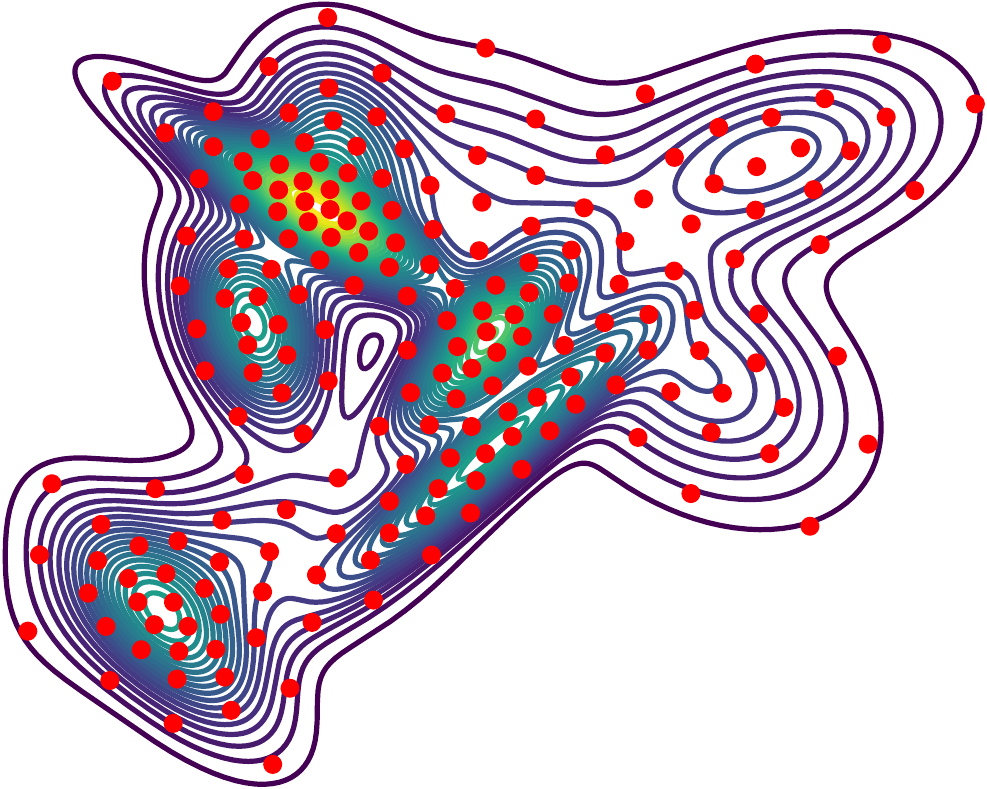}\\
    \includegraphics[width=.21\textwidth]{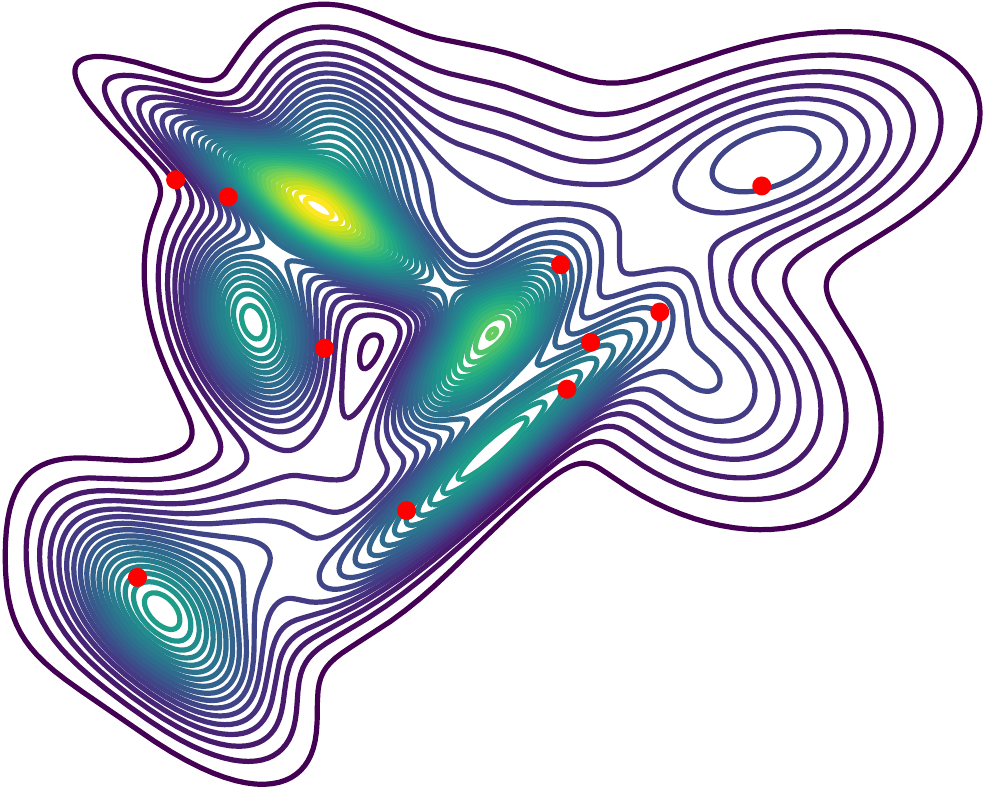} &
    \includegraphics[width=.21\textwidth]{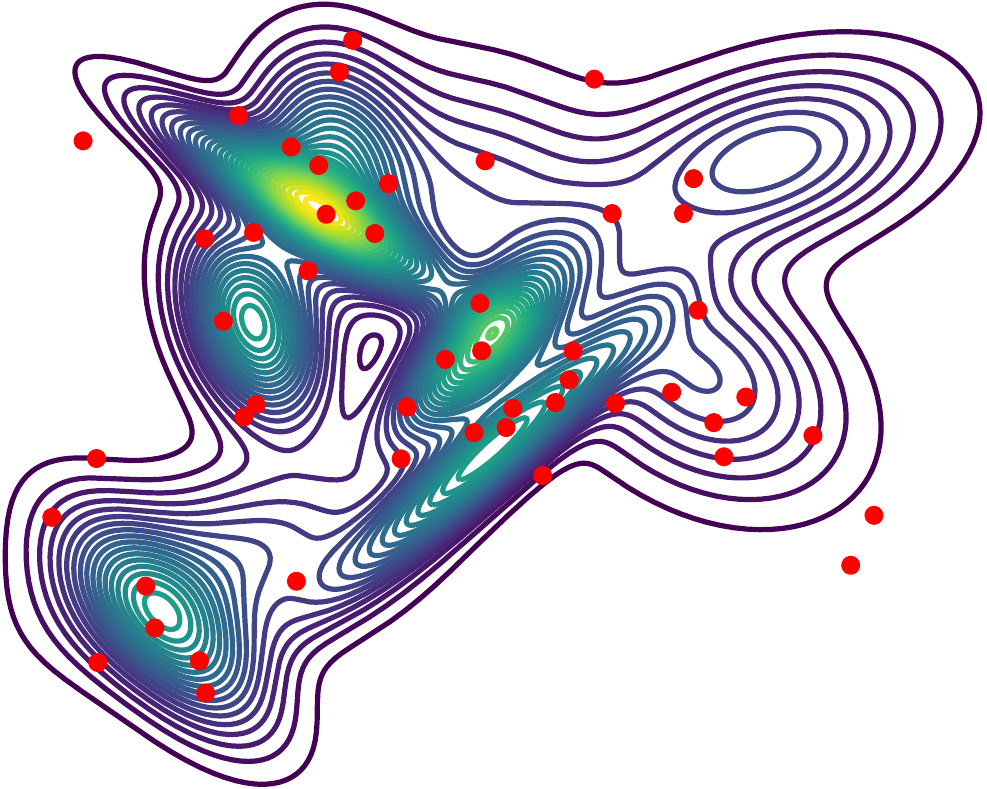} &
    \includegraphics[width=.21\textwidth]{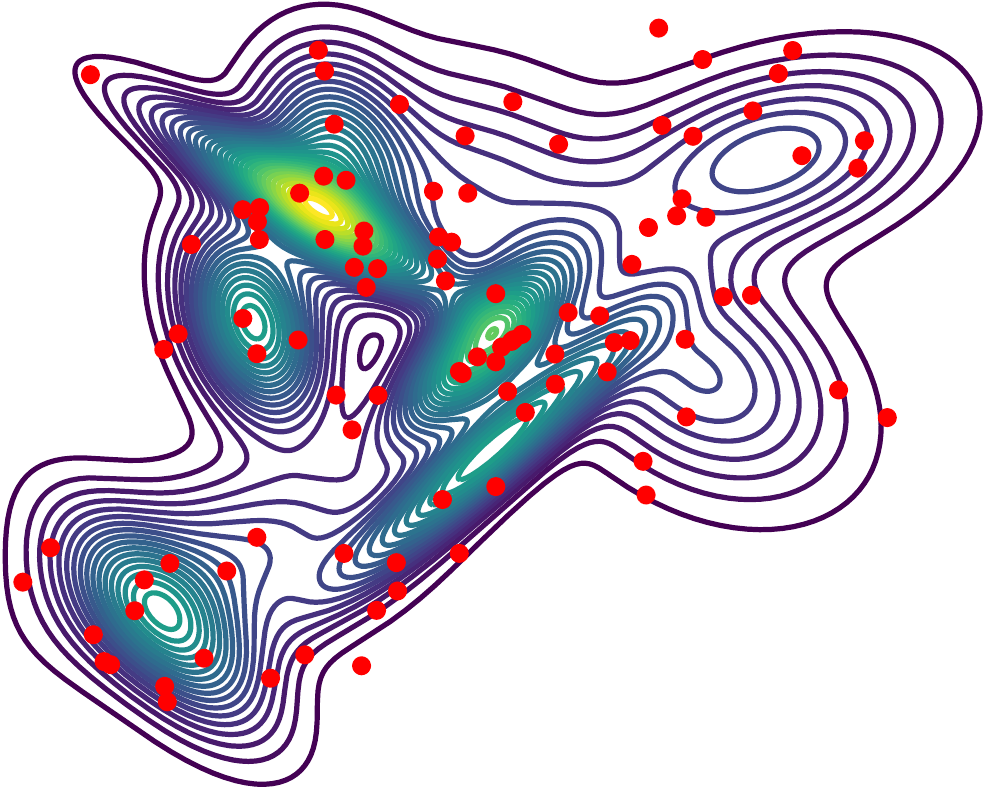} &
    \includegraphics[width=.21\textwidth]{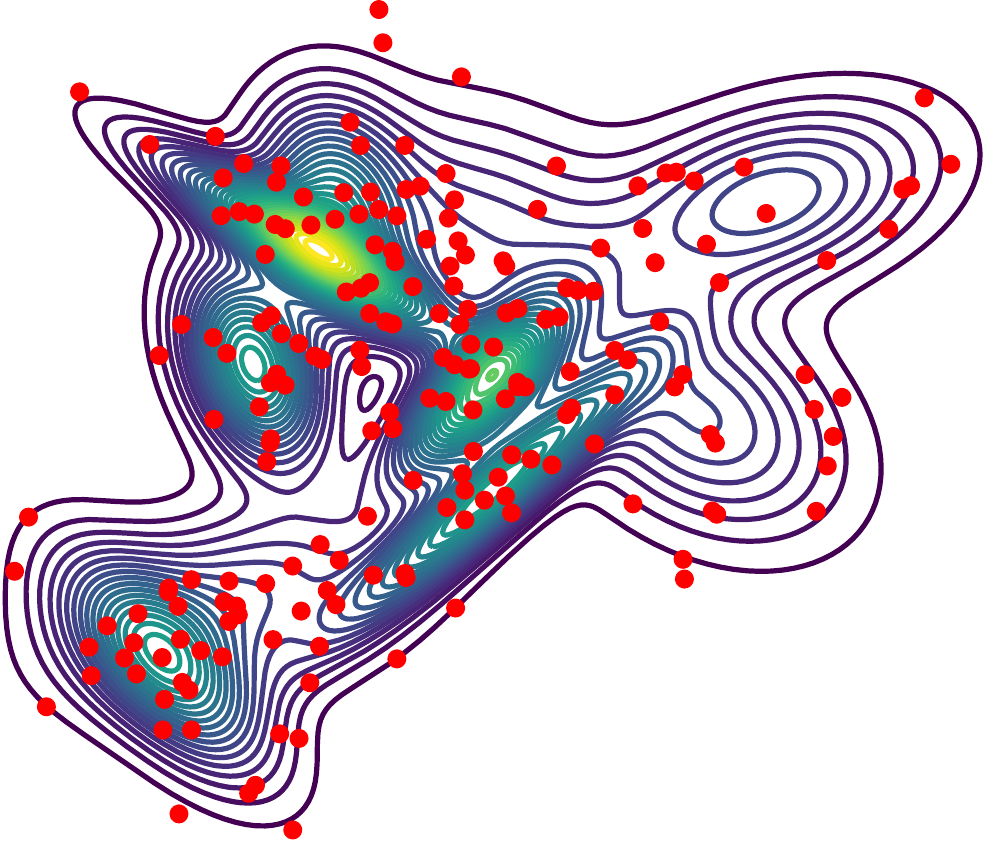}
  \end{tabular}
  \caption{The $n$ point approximation of a mixture of ten Gaussians. Top row: our method with 10, 50, 100, and 200 points. Bottom row: iid sampling with the same number of points.}
  \label{fig:mixture}
\end{figure*}

%%% Local Variables:
%%% mode: latex
%%% TeX-master: "icml2018_barycenters"
%%% End:

% !TEX root = icml2018_barycenters.tex

\section{Experiments}

\begin{figure}[!t]
  \centering
  \begin{tabular}{@{}c@{}}
    \includegraphics[width=.5\columnwidth]{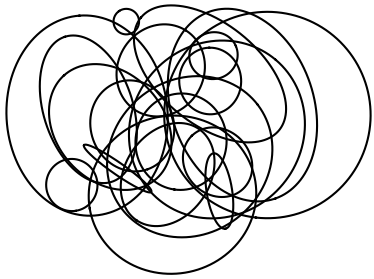}
  \end{tabular}
  \begin{tabular}{@{}c@{\hskip .2\columnwidth}c@{}}
    \includegraphics[width=.35\columnwidth]{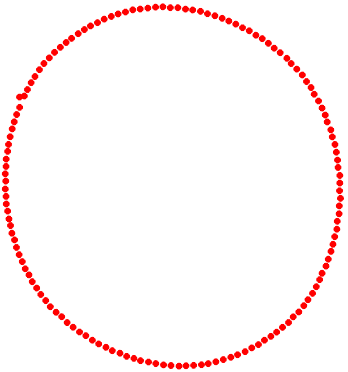} &
    \includegraphics[width=.35\columnwidth]{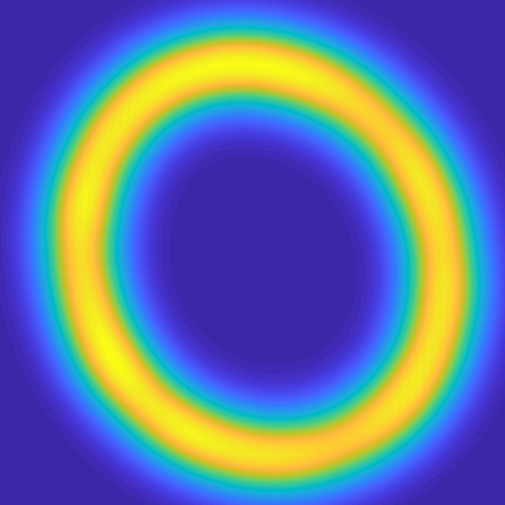}\\
    (a) & (b)
  \end{tabular}
  \caption{Barycenter of randomly generated ellipses. Top: plot showing 20 ellipses with randomly drawn center, semi-major and semi-minor axes, and skew. Bottom: (a) The output of our algorithm is a sharp distribution approximating a circle. (b) The output of~\cite{solomon_convolutional_2015} with a regularizer value of $\gamma = 0.1$.}
  \label{fig:ellipse}
\end{figure}

\begin{figure}[!t]
  \centering
  \begin{tabular}{@{}c@{\hskip .15\columnwidth}c@{}}
    \includegraphics[width=.4\columnwidth]{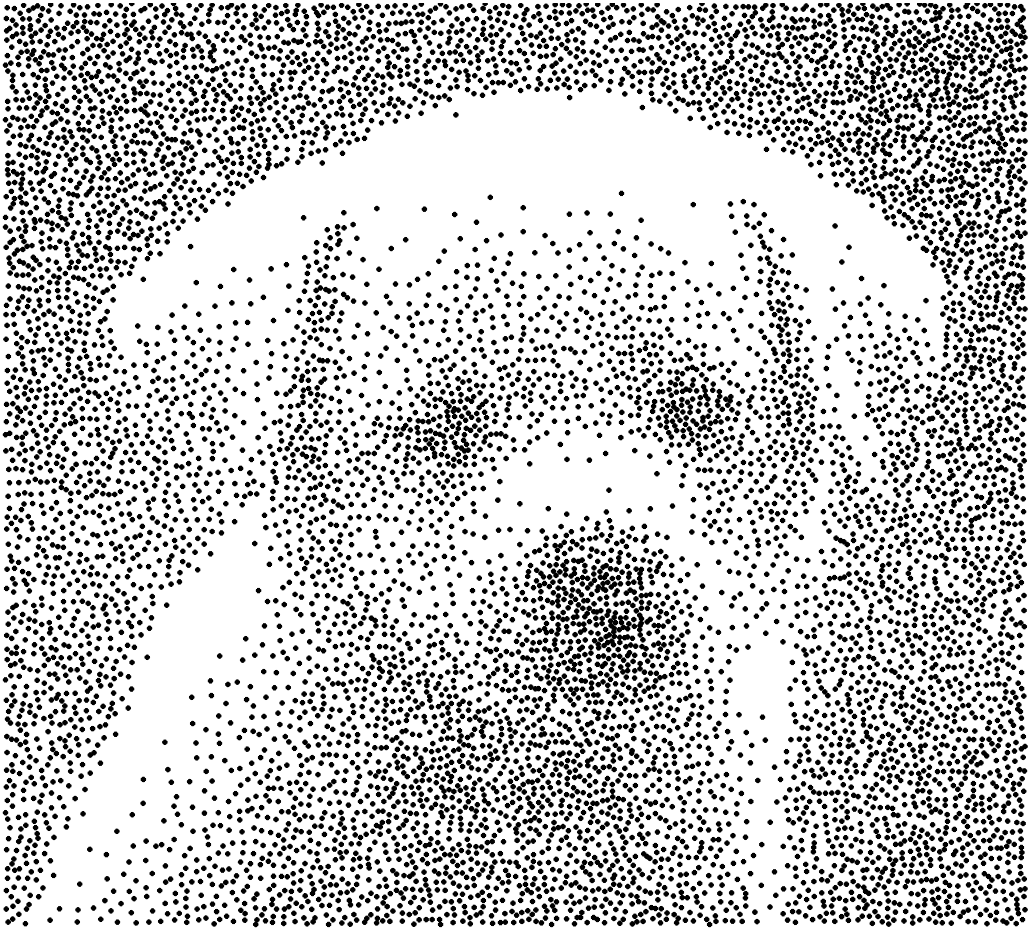} &
    \includegraphics[width=.4\columnwidth]{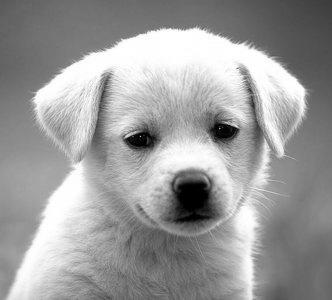}
  \end{tabular}
  \caption{Blue noise sampling. Left: 10K samples from our algorithm. Right: Original image (approximately 90K pixels).}
  \label{fig:bluenoise}
\end{figure}

We showcase the versatility of our method on several applications. We typically use between 16K and 256K samples per input distribution to approximate the power cell density and barycenter. The variance is due to different problem sizes and dimensionality of the input measures. We stop the gradient ascent step when $\|\nabla F\|_2^2 \leq 10^{-6}$. The snap step empirically converges in under 20 iterations, and several of our examples use only one step.

\subsection{Distributions with Sharp Features}

Our algorithm is well-suited to problems where the input distributions have very sharp features. We test against the algorithms in \cite{staib2017parallel} and \cite{solomon_convolutional_2015} on two test cases: ten uniform distributions over lines in the 2D plane (Figure~\ref{fig:sharp}), and 20 uniform distributions over ellipses (Figure~\ref{fig:ellipse}).

The results of Figures \ref{fig:sharp} and \ref{fig:ellipse} show that our barycenter is more sharply supported than the results of competing methods. Our output agrees with that of~\citet{solomon_convolutional_2015}, but our results more closely match expected behavior. We strongly suspect that the true barycenter in Figure~\ref{fig:sharp} is also a uniform measure on a line, while that in Figure~\ref{fig:ellipse} is a circle centered at the origin.

\subsection{The Case $N = 2$}

In the case of two input measures $\mu_1$ and $\mu_2$, we expect the barycenter to be McCann's interpolant~\cite{agueh_barycenters_2011,mccann1997convexity}:
\begin{align*}
  \mu_{1/2} \eqdef \left(\frac{1}{2}\text{id} + \frac{1}{2}T\right)_{\#}\mu_0 = \left(\frac{1}{2} \text{id} + \frac{1}{2}T^* \right)_{\#}\mu_1
\end{align*}
where $T$ is the optimal map, and $T^*$ is the inverse map, while $\#$ denotes the pushforward of a measure.

We test this on two uniform distributions on the unit square in Figure~\ref{fig:mccann}. The transport map in this case is transport of the entire distribution along a straight line. As expected from McCann's interpolant, we recover a uniform distribution on the unit square halfway between the two input distributions. We show our results alongside those of \cite{staib2017parallel}. Notice that their output barycenter is not uniform, and that it has non-zero measure outside the true barycenter.

\subsection{The Case $N = 1$}

The case $N = 1$ bears interest as well. There are instances when sampling iid from a distribution yields samples that do not approximate the underlying distribution accurately. We showcase two applications in generating super samples from distributions, as well as approximating grayscale images through blue noise.

\subsubsection{Blue Noise}

The term blue noise refers to an unstructured but even and isotropic distribution of points. It has been used in image dithering as it captures image intensity via local point density, without the need for varying point sizes as in halftoning.

\citet{de2012blue} describe the link between optimal transport and blue noise generation. We recover a stochastic version of their algorithm by taking $\mu$ a discrete distribution over the image pixels proportional to intensity. As our method is more general, we observe performance loss, but the output is of comparable quality (Figure~\ref{fig:bluenoise}).

\subsubsection{Super Samples}

Our method can be adapted to generate super samples from complex distributions~\cite{DBLP:conf/uai/ChenWS10}. Figure~\ref{fig:mixture} details our results on a mixture of ten Gaussians. Our method better approximates the shape of the underlying distribution due to negative autocorrelations: points move away from oversampled regions. The points drawn iid from the mixture tend to oversample around the larger modes and do not approximate density contours as well.

%%% Local Variables:
%%% mode: latex
%%% TeX-master: "icml2018_barycenters"
%%% End:

% !TEX root = icml2018_barycenters.tex

\section{Conclusion}

We have proposed an algorithm for computing the Wasserstein barycenter of continuous measures using only samples from the input distributions. The algorithm decomposes into a concave maximization and a fixed point iteration similar to the mean-shift and $k$-means algorithms. Our algorithm is easy to implement and parallelize, and it does not rely on a fixed-support grid. This allows us to recover much sharper approximations to the barycenter than previous methods. Our algorithm is general and versatile enough to be applied to other problems beyond barycenter computation.

There are several avenues for future work. Solving the concave maximization problem is currently a bottleneck for our algorithm as we do not have access to the function value or the Hessian, but we believe multiscale methods can be adapted to our approach. The potential applications of this method extend beyond what was covered. One application we highlight is in developing coresets that minimize the distance to the empirical distribution on the input data.

%%% Local Variables:
%%% mode: latex
%%% TeX-master: "icml2018_barycenters"
%%% End:

\subsection*{Acknowledgements}
The authors thank Fernando de Goes, Marco Cuturi, Gabriel Peyr\'e, and Matthew Staib for input and early discussions.  The authors acknowledge the generous support of Army Research Office grant W911NF-12-R0011 (``Smooth Modeling of Flows on Graphs''), from the MIT Research Support Committee, from the MIT--IBM Watson AI Lab, from the Skoltech--MIT Next Generation Program, and from an Amazon Research Award.

\bibliography{barycenters}
\bibliographystyle{icml2018}

\end{document}